\documentclass[hidelinks,onefignum,onetabnum]{siamart250211}
\usepackage{amssymb}
\usepackage{booktabs}   
\usepackage{multirow}   

\usepackage{lipsum}
\usepackage{amsfonts}
\usepackage{graphicx}
\usepackage{epstopdf}
\usepackage{algorithmic}
\ifpdf
  \DeclareGraphicsExtensions{.eps,.pdf,.png,.jpg}
\else
  \DeclareGraphicsExtensions{.eps}
\fi

\newsiamremark{remark}{Remark}
\newsiamremark{hypothesis}{Hypothesis}
\crefname{hypothesis}{Hypothesis}{Hypotheses}
\newsiamthm{claim}{Claim}
\newsiamremark{fact}{Fact}
\crefname{fact}{Fact}{Facts}

\headers{BlinDNO for Time-Label-Free Dynamical System Reconstruction}{Zhijun Zeng, Junqing Chen, and Zuoqiang Shi}

\title{BlinDNO: A Distributional Neural Operator for Dynamical System Reconstruction from Time-Label-Free data\thanks{Submitted to the editors DATE.
\funding{This work was supported by the National Natural Science Foundation of China (NSFC) 92370125 and New Cornerstone Investigator Program 100001127}}}

\author{Zhijun Zeng\thanks{Department of Mathematical Sciences, Tsinghua University, Beijing, China 
  (\email{zengzj22@mails.tsinghua.edu.cn}).}
\and Junqing Chen\thanks{Department of Mathematical Sciences, Tsinghua University, Beijing, China  
  (\email{jqchen@tsinghua.edu.cn}).}
\and Zuoqiang Shi\thanks{Yau Mathematical Sciences Center, Tsinghua University, Beijing, China and Yanqi Lake Beijing Institute of Mathematical Sciences and Applications, Beijing, China  
  (\email{zqshi@tsinghua.edu.cn}).Corresponding author.}}
\usepackage{amsopn}

\ifpdf
\hypersetup{
  pdftitle={BlinDNO: A Distributional Neural Operator for Dynamical System Reconstruction from Time-Label-Free data},
  pdfauthor={Zhijun Zeng, Junqing Chen, and Zuoqiang Shi}
}
\fi

\begin{document}

\maketitle

\begin{abstract}
We study an inverse problem for stochastic and quantum dynamical systems in a time-label-free setting, where only unordered density snapshots—sampled at unknown times drawn from an observation-time distribution $\nu_t$—are available. These observations induce a distribution over state densities, from which we seek to recover the parameters of the underlying evolution operator. We formulate this as learning a distribution-to-function neural operator and propose \textsc{BlinDNO}, a permutation-invariant architecture that integrates a multiscale U-Net encoder with attention-based mixer.
Numerical experiments on a wide range of stochastic and quantum systems—including a 3D protein-folding mechanism reconstruction problem in cryo-EM setting—demonstrate that \textsc{BlinDNO} reliably recovers governing parameters and consistently outperforms existing neural inverse operator baselines.
\end{abstract}



\section{Introduction}
Stochastic dynamics are fundamental to advancing our understanding of natural phenomena. By leveraging massive datasets to uncover the underlying stochastic equations that govern complex physical systems, we can substantially enhance our capacity to model and predict system behavior across diverse scientific disciplines\cite{DMKEBSK23Chaos,schmidt2009distilling,chen2024constructing}. In the classical stochastic system identification problem, one collects either trajectory data~$\{\{t_i^j, x_i^j\}_{i=1}^n\}_{j=1}^m$~or aggregate measurements~$\{t_i, \{x_i^j\}_{j=1}^m\}_{i=1}^n$~and then employs Bayesian inference, sparse regression, weak‐form methods, or machine‐learning techniques to extract the evolution law~$x_{t+1} = f(x_t, t, \omega)$\cite{opper2019variational,chen2021solving,messenger2022learning,zhu2024dyngma,lu2024weak}. The effectiveness of these approaches, however, hinges on the availability of precise time labels \(t_i\), whose acquisition may be hindered by technical constraints in many practical settings. When exact time stamps are unavailable and only a distribution of observation times is known, a novel inverse problem emerges: reconstructing the stochastic dynamics without time label.

In this setting, the probability density function of an observable state evolves according to the probabilistic transport equation
\begin{equation}\label{eqn_intro}
  \frac{\partial \rho(x,t)}{\partial t} = \mathcal{L}\,\rho(x,t),
\end{equation}
where \(\mathcal{L}\) denotes a suitable probability evolution operator(e.g. Fokker-Planck operator). Suppose that the observation times \(t_i\) are sampled from a known distribution \(P(t)\), and that the initial condition is given by \(\rho(x,0)=\rho_0(x)\).The available data then consist of sampled density functions
\(
  \{\rho(x, t_i)\colon t_i \sim P(t)\}_{i=1}^M.
\)
The objective is to infer the operator \(\mathcal{L}\) directly from these distributions. This problem arises in various biological applications, such as reconstructing conformational transitions in single‐particle cryo‐EM and inferring developmental trajectories from single‐cell RNA sequencing data\cite{rubinstein2015alignment,papasergi2024time,gilles2025cryo,grun2015single,saelens2019comparison}. For instance,  single-particle cryo-EM reconstructs 3D Coulomb-potential distributions of biomolecules at near–atomic resolution from thousands of static 2D particle images. Biomolecules exhibit conformational heterogeneity, and their functions are often governed by transitions among different stable states  $\rho_{\mathrm{start}}(x)$ and $\rho_{\mathrm{end}}(x)$. Several recent works have leveraged optimal-transport–guided dynamics to elucidate conformational transitions\cite{ecoffet2020morphot,ecoffet2021application}, while others have exploited optimal transport within neural networks to approximate more complex transport forces in density-function or latent spaces\cite{fan2024cryotrans,punjani20233dflex}. Yet, these methods yield only deformation trajectories and do not recover the underlying energy landscape or force field. This limitation constrains deeper mechanistic insights into molecular function.

Deep learning has recently emerged as an efficient surrogate for both forward and inverse maps in dynamical systems. Operator learning frameworks seek to infer maps defined by partial differential equations—such as parameter-to-solution or observation-to-parameter operators—directly from data. Prominent among these are Deep Operator Networks (DeepONets)\cite{lu2021learning}, which employ a “branch–trunk” architecture to decouple input functions from spatial evaluation points, and Fourier Neural Operators (FNOs)\cite{li2020fourier}, which leverage FFT-based convolutions to capture nonlocal and nonlinear interactions. Further variants enhance predictive accuracy by incorporating attention mechanisms or multigrid-inspired refinements\cite{wutransolver,he2023mgno,hao2023gnot}. In inverse settings, many studies construct neural surrogates—based on DeepONets, FNOs, and other architectures such as CNNs or UNets—to approximate solution-to-parameter or operator-to-parameter maps\cite{wu2019inversionnet,zhu2023fourier,liu2025neumann,zeng2025vivo}. However, these approaches typically require deterministic, discretized inputs rather than distributions. To address this, Roberto et al.\cite{molinaro2023neural} introduced the Neural Inverse Operator (NIO), which enforces input permutation invariance via feature averaging and employs randomized batching to learn a distribution-to-parameter map. Despite its innovations, NIO remains limited by the representational capacity of its DeepONet part and by information loss incurred during averaging.

Motivated by these challenges and existing methods, we introduce \textsc{BlinDNO}, a novel neural operator framework that leverages deep learning for time-label-free dynamical-system reconstruction. Our main contributions are as follows:
\begin{enumerate}
    \item We provide a rigorous definition of the general learning problem for reconstructing dynamical systems without time label and and decompose the resulting distributional neural operator into three components—an imaging operator, a feature-fusing operator, and a refinement operator. Building on this analysis, we propose \textsc{BlinDNO} to approximate distribution-to-parameter maps by integrating an attention-based feature-fusing module into a U-Net–based imaging operator, underpinned by an FNO backbone to enhance high-frequency detail recovery.
    \item We validate \textsc{BlinDNO} on a suite of problems governed by classical and quantum mechanics, demonstrating accurate reconstructions across varying dimensions, complex dynamics, and non-conservative force fields. Our method further accommodates arbitrary observation-time distributions.
    \item We apply \textsc{BlinDNO} to a realistic three-dimensional protein-folding example drawn from cryo-EM, illustrating its potential to reveal molecular mechanisms in practical structural-biology settings.
\end{enumerate}
The rest of the paper is organized as follows. In Section~\ref{sec:problem}, we present the mathematical formulation for the time-label-free dynamical system reconstruction problem. Section~\ref{sec:method} presents the discussion of a neural operator from distribution to functions and the proposed \textsc{BlinDNO} architecture. Section~\ref{sec:result} showcases numerical results. Finally, Section~\ref{sec:conclusion} provides conclusions of this work and gives some discussions of future works.

\section{Problem Formulation}\label{sec:problem}
Here, we consider an observable state defined on a domain 
\(\Omega \subset \mathbb{R}^d\), whose evolution is governed by 
a probabilistic transport equation parameterized by \(\theta^{\star}\):
\begin{equation}\label{eqn_main}
\begin{cases}
\displaystyle \frac{\partial \rho(x,t)}{\partial t} 
   \;=\; \mathcal{L}_{\theta^{\star}}\,\rho(x,t),\\[1ex]
\rho(x,0)=\rho_0(x),
\end{cases}
\end{equation}
where \(\mathcal{L}_{\theta^{\star}}\) denotes a transport operator dictated 
by the underlying physical principles. 

Let \(\mathcal{H}\) denote the infinite-dimensional space of 
probability density functions on \(\Omega \subset \mathbb{R}^d\). 
The dynamics specified in \eqref{eqn_main} induce an 
observation operator
\begin{equation}\label{eqn:observation_operator}
\Lambda_{\theta^{\star}} : [0,T]\;\longrightarrow\;\mathcal{H},
\quad
t \;\mapsto\; \rho(\cdot,t),
\end{equation}
which assigns to each observation time \(t\) the corresponding 
state density. Suppose that the observation times are sampled from a prescribed 
distribution \(\nu_t\in\mathcal{P}([0,T])\) supported on \([0,T]\), referred to as the 
observation-time distribution. The pushforward measure
\[
\nu_{\rho}^{\theta^{\star}} \;=\;\Lambda_{\theta^{\star}}\#\nu_t
\]
then characterizes the induced data distribution.  

In practice, we observe an unordered collection of samples,
\(
\bigl\{\rho(\cdot,t_i)\colon t_i\sim\nu_t\bigr\}_{i=1}^N,
\)
which we identify with the empirical probability measure 
\(\nu_{\rho}^{\theta^{\star},N} \in \mathcal{P}(\mathcal{H})\) defined by
\begin{equation}
  \nu_{\rho}^{\theta^{\star},N}  \;:=\;\frac{1}{N}\sum_{i=1}^N 
    \delta_{\rho(\cdot,t_i)}\approx\nu_{\rho}^{\theta^{\star}},
\end{equation}
where \(\delta_{y}\) denotes the Dirac measure, i.e., 
\(\delta_{y}(A)=1\) if \(y\in A\) and \(\delta_{y}(A)=0\) otherwise.  
This empirical measure serves as a discrete approximation of 
\(\nu_{\rho}^{\theta^{\star}}\). Our objective is to find the unknown parameter \(\theta\) such that the induced distribution \(\nu_{\rho}^{\theta}\) is sufficiently close to the observed data distribution \(\nu_{\rho}^{\theta^\star}\).  Let \(\mathcal{L}_{\theta^\star}\) denote the true evolution operator; then the parameter estimation problem can be formulated as
\begin{equation}\label{eqn_opt_problem}
\begin{aligned}
  \min_{\theta\in\Theta}\quad & d\bigl(\nu_{\rho}^{\theta},\,\nu_{\rho}^{\theta^\star}\bigr) \\
  \text{subject to}\quad &
  \begin{cases}
    \nu_{\rho}^{\theta} = \Lambda_{\theta}\#\nu_t,\\
    \nu_{\rho}^{\theta^\star} = \Lambda_{\theta^\star}\#\nu_t,
  \end{cases}
\end{aligned}
\end{equation}
where \(d(\cdot,\cdot)\) quantifies the discrepancy between probability distributions.
\paragraph{\textbf{Problem I: SDE driven dynamic}}
Consider the observable state $X_t\in\mathbb{R}^d$ evolves by the following stochastic differential equation:
\begin{equation}\label{eqn:sde}
   dX_t = \mu(X_t) \, dt + \sigma(X_t) \, dW_t,
\end{equation}
where $\mu(X_t)\in$ is the vector drift term, $\sigma(X_t)\in \mathbb{R}^{d\times d}$ is the diffusion term, $W_t$ is the Brownian motion in $\mathbb{R}^d$. Then the density function $\rho(x, t)$ of the above state $X_t$ can be described by the Fokker-Planck equation (FPE)\cite{risken1996fokker}, and we restate the result in Lemma \ref{lem:fp}.
\begin{lemma}\label{lem:fp}
    Suppose $X_t$ solves the SDE \eqref{eqn:sde},  then the probability density function $\rho(x,t)$ satisfies the following d-dimensional Fokker-Planck equation by the It\^{o} integral
    \begin{equation}\label{eqn:fp}
    \frac{\partial \rho}{\partial t} = -\nabla\cdot(\boldsymbol{\mu} \rho) + \sum_{i,j}^d \partial_{ij}(D_{ij}\rho),
    \end{equation}
    where $t\in[0,T]\subset\mathbb{R}$, 
\(
\mu = \bigl[\mu_1(x,t),\,\mu_2(x,t),\,\dots,\,\mu_d(x,t)\bigr]^T,
\)
and the diffusion matrix $[D_{ij}]=[D_{ij}(x,t)]$ is given by
\begin{equation}\label{eqn:diffusion}
D = \frac{1}{2}\,\sigma\,\sigma^T.
\end{equation}
\end{lemma}
In this case, the probabilistic evolution equation is the FPE \eqref{eqn:fp} and we need to recover the unknown drift term $\boldsymbol{\mu}$ and diffusion matrix $D$ given the collected samples $\bigl\{\rho(\cdot,t_i)\colon t_i\sim\nu_t\bigr\}_{i=1}^N$.

\paragraph{\textbf{Problem II: Quantum-mechanical dynamic}}
Consider a quantum state \(\Psi(x,t)\in L^2(\mathbb{R}^d)\) evolving under the time‐dependent nonlinear Schrödinger equation in the form
\begin{equation}\label{eqn_schrodinger}
i\hbar\frac{\partial}{\partial t}\Psi(x,t)
=\biggl[-\frac{\hbar^2}{2m}\nabla^2 + V(x,t)
+ g\bigl(\alpha,|\Psi(x,t)|^2\bigr)\biggr]\Psi(x,t),
\quad
\Psi(x,0)=\Psi_0(x)\,,
\end{equation}
where \(V(x,t)\) is the time–dependent potential, \(\alpha\) parametrizes the nonlinear coupling \(g\), and \(\hbar\) is the reduced Planck constant.  The observable probability density is
\[
\rho(x,t)=\bigl|\Psi(x,t)\bigr|.
\]
Similarly, we can define the observation operator by
\[
\Lambda_{(V,\alpha)}:[0,T]\longrightarrow\mathcal{H},
\quad
t\mapsto \rho(\cdot,t)\,,
\]
and assume observation times \(t_i\sim\nu_t\).  We need to recover recover the unknown potential \(V\) and nonlinearity parameter \(\alpha\) given the collected data
\(
\bigl\{\rho(\cdot,t_i)\colon t_i\sim\nu_t\bigr\}_{i=1}^N.
\)

In both the SDE–driven and 
quantum‐mechanical dynamical settings, the forward problem is to 
determine the induced data distribution given the initial density 
\(\rho_0(x)\) or state \(\Psi(x,0)\), the parameters 
\(\theta \in \Theta\), and the observation‐time distribution 
\(\nu_t \in \mathcal{P}([0,T])\).  
This formulation naturally defines the forward operator
\begin{equation}\label{eqn:forward}
\mathcal{F}:\Theta\;\longrightarrow\;\mathcal{P}(\mathcal{H}),
\qquad
\theta\;\mapsto\;\Lambda_{\theta}\#\nu_t.
\end{equation}

To the best of our knowledge, the optimization problem \eqref{eqn_opt_problem} has not been addressed in the literature, owing to the inherent difficulty of comparing complex probability measures in the objective functional.  Rather than solving the optimization task\eqref{eqn_opt_problem} , an alternative paradigm—direct inversion—aims to approximate the inverse map
\begin{equation}\label{eqn:inverse_map}
\mathcal{F}^{-1}:\mathcal{P}(\mathcal{H})\;\longrightarrow\;\Theta,
\qquad
\nu_{\rho}^{\star}\;\mapsto\;\theta.
\end{equation}
Establishing rigorous guarantees for the existence and the uniqueness of this inverse map constitutes the principal challenge of the problem.  Inspired by recent data‐driven approaches in PDE inverse problems, which approximate inverse operators via deep learning, our work seeks to extend this framework to the recovery of parameters from distributional observations.
\section{Method}\label{sec:method}
\subsection{Neural Operators between Function Spaces}\label{sec:no_function}
Let \(\Omega\subset\mathbb{R}^d\) be  bounded open set, and consider the Banach spaces
  $\mathcal{X} = \mathcal{X}(\Omega;\mathbb{R}^{d_a})$ and$ 
  \Theta = \Theta(\Omega;\mathbb{R}^{d_u})$
of input and output functions defined on bounded Euclidean subsets $\Omega$.  We assume that there exists a target operator that we wish to learn
\[
  \mathcal{F}\colon \mathcal{X} \;\longrightarrow\; \Theta,
  \qquad
  \theta = \mathcal{F}(\rho),
\]
arising, for instance, from a parametric PDE.  In the supervised operator learning setting, given a finite dataset
\(
  \{(\rho^{(i)},\theta^{(i)})\}_{i=1}^N \subset \mathcal{X}\times\Theta,
\)
the learning task is to construct a data-driven surrogate operator
\[
  \mathcal{G}_\phi\colon \mathcal{X} \;\longrightarrow\; \Theta,
\]
by solving the \emph{empirical risk minimization} (ERM) problem
\[
  \min_{\phi\in\Phi}
  \frac{1}{N}\sum_{i=1}^N 
  \bigl\lVert \theta^{(i)} - \mathcal{G}_\phi(\rho^{(i)})\bigr\rVert_{\Theta}^2,
\]
where \(\Phi\) denotes the hypothesis space of admissible surrogate operators. Below we briefly summarize two representative architectures: the Deep Operator Network and the Spectral Neural Operator.

\subsubsection{Deep Operator Network (DeepONet)}
The DeepONet\cite{lu2021learning} architecture realizes \(\mathcal{G}_\phi\) through an encoder–decoder structure consisting of two subnetworks. Let \(\mathcal{E} = (L_j)_{j=1}^{d_\beta} \subset \mathcal{L}(\mathcal{X};\mathbb{R})\) denote a finite collection of continuous linear functionals on \(\mathcal{X}\), which act as probes of the input function \(\rho\). 
For example, if \(\mathcal{X}\) is a Hilbert space, each \(L_j\) may be taken as a projection onto a basis element. The \emph{branch net} \(\beta\colon \mathbb{R}^{d_\beta}\to\mathbb{R}^p\) processes the encoded inputs
\[
  \mathcal{E}(\rho) = \bigl(L_1(\rho),\dots,L_{d_\beta}(\rho)\bigr)\in \mathbb{R}^{d_\beta},
\]
producing output coefficients \(\beta(\mathcal{E}(\rho))\in\mathbb{R}^p\).  

In parallel, the \emph{trunk net} \(\tau\colon \Omega \to \mathbb{R}^p\) evaluates \(p\) basis functions \(\tau_1,\dots,\tau_p\) at a query location \(y\in\Omega\).  
The decoder then forms the approximation
\begin{equation}\label{eq:DeepONet}
  \bigl(\mathcal{G}_\phi^{\mathrm{DON}}(\rho)\bigr)(y)
  \;=\; \sum_{k=1}^p \beta_k\!\bigl(\mathcal{E}(\rho)\bigr)\,\tau_k(y),
  \qquad \rho\in\mathcal{X}, \; y\in\Omega,
\end{equation}
where \(\phi\) collects the parameters of both \(\beta\) and \(\tau\).  
In this formulation, the branch net supplies the coefficients while the trunk net provides location-dependent basis functions, and their combination yields an approximation in \(\Theta\).

\subsubsection{Spectral Neural Operator}
Let the \emph{channel dimension} be \(d_c \ge \max(d_a,d_u)\) and define the latent
function space
\(
  \mathcal{H}=\mathcal{H}(\Omega;\mathbb{R}^{d_c}).
\)
For convenience we set the hidden-layer spaces
\(
  V_t:=\mathcal{H}\ (t=0,1,\dots,T)
\).
A spectral neural operator realizes
\(\mathcal{G}\colon\mathcal{X}\to\Theta\) as
\[
  \mathcal{G}
  \;=\; \mathcal{Q} \circ \mathcal{L}_T \circ \cdots \circ \mathcal{L}_1 \circ \mathcal{P},
\]
with pointwise \emph{lifting} operator $\mathcal{P}:\mathcal{X}\rightarrow\mathcal{H}$  and \emph{projection} operator
$\mathcal{Q}:\mathcal{H}\rightarrow\mathcal{\Theta}$ defined as
\[
  (\mathcal{P}\rho)(x) := P(x,\rho(x)) \in \mathbb{R}^{d_c},
  \qquad
  (\mathcal{Q}h)(x) := Q(x,h(x)) \in \mathbb{R}^{d_u},
\]
where \(P:\Omega\times\mathbb{R}^{d_a}\to\mathbb{R}^{d_c}\) and \(Q:\Omega\times\mathbb{R}^{d_c}\to\mathbb{R}^{d_u}\)
are typically shallow neural networks.
Each hidden layer \(\mathcal{L}_t:V_{t-1}\to V_t\) acts as
\begin{equation}\label{eq:SNO_layer}
  (\mathcal{L}_t h)(x)
  \;=\;
  \sigma_t\!\Bigl(
      W_t\,h(x) \;+\; (\mathcal{K}_t h)(x) \;+\; b_t(x)
  \Bigr),
  \qquad x\in\Omega,
\end{equation}
where \(W_t\in\mathbb{R}^{d_c\times d_c}\), \(b_t\in\mathcal{H}\),
\(\sigma_t:\mathbb{R}\to\mathbb{R}\) is a nonlinearity acting componentwise on functions,
and \(\mathcal{K}_t:\mathcal{H}\to\mathcal{H}\) is the (generally nonlocal) kernel integral operator
\begin{equation}\label{eq:SNO_kernel}
  (\mathcal{K}_t h)(x)
  \;=\; \int_{\Omega} \kappa_t(x,y)\,h(y)\,\mathrm{d}y,
  \qquad
  \kappa_t:\Omega\times\Omega\to\mathbb{R}^{d_c\times d_c}.
\end{equation}
Different parameterizations of the matrix–valued kernel \(\kappa_t\) yield different
neural–operator architectures.

\medskip
\noindent\emph{Fourier Neural Operator (FNO).}
Assume \(\Omega=\mathbb{T}^d=[0,1]^d_{\mathrm{per}}\) and parameterize \(\mathcal{K}_t\) in Fourier space.
Let \(\widehat{h}(k)\in\mathbb{C}^{d_c}\) be the vector of Fourier coefficients of \(h\),
and let \(\Lambda_{k_{\max}}:=\{k\in\mathbb{Z}^d:\|k\|_\infty\le k_{\max}\}\) denote the retained modes.
The FNO layer specifies a (truncated) Fourier multiplier \(P_t^{(k)}\in\mathbb{C}^{d_c\times d_c}\) and acts by
\begin{equation}\label{eq:FNO_multiplier}
  \widehat{(\mathcal{K}_t h)}(k)
  \;=\;
  \begin{cases}
    P_t^{(k)}\,\widehat{h}(k), & k\in\Lambda_{k_{\max}},\\[0.3ex]
    0, & \text{otherwise},
  \end{cases}
\end{equation}
equivalently,
\begin{equation}\label{eq:FNO_modes}
  \bigl[(\mathcal{K}_t h)(x)\bigr]_{\ell}
  \;=\;
  \sum_{k\in\Lambda_{k_{\max}}}\;
  \sum_{j=1}^{d_c}
  \bigl(P_t^{(k)}\bigr)_{\ell j}\,
  \big\langle e^{2\pi i \langle k,\cdot\rangle},\, h_j \big\rangle_{L^2(\mathbb{T}^d;\mathbb{C})}\,
  e^{2\pi i \langle k,x\rangle},
  \qquad \ell=1,\dots,d_c.
\end{equation}
Thus \(\mathcal{K}_t\) is a translation–invariant convolution operator
whose matrix–valued kernel \(\kappa_t(x-y)\) has Fourier coefficients \(P_t^{(k)}\);
the computation is implemented efficiently via FFTs.

\medskip
\noindent\emph{Convolution–based and Graph Neural Operators.}
An alternative parameterization of the kernel operator \eqref{eq:SNO_kernel} 
is obtained by localizing the integration to a neighborhood of radius \(r>0\). 
Specifically,
\begin{equation}\label{eq:CNO_cont}
  (\mathcal{K}_t h)(x)
  \;=\;\int_{\Omega}\kappa_t(x,y)\,h(y)\,\mathrm{d}y
  \;\approx\;\int_{B_r(x)}\kappa_t(x,y)\,h(y)\,\mathrm{d}y,
\end{equation}
where \(B_r(x)\subset \Omega\) denotes the ball of radius \(r\) centered at \(x\).  A numerical approximation of \eqref{eq:CNO_cont} can be obtained by discretization.
Let \(\{y_i\}_{i=1}^M\subset B_r(x)\) denote the neighboring points of \(x\) and 
\(\mu(y_i)\) the associated quadrature weights. Then
\begin{equation}\label{eq:CNO_discrete}
  (\mathcal{K}_t h)(x) 
  \;\approx\; \sum_{i=1}^M \kappa_t(x,y_i)\,h(y_i)\,\mu(y_i).
\end{equation}
This localized formulation admits efficient realizations:  
if \(\Omega\) is represented by a regular grid, 
\eqref{eq:CNO_discrete} corresponds to a standard convolution (\emph{convolution-based neural operator}\cite{raonic2023convolutional}); 
if \(\Omega\) is discretized as an irregular mesh or point cloud, 
\eqref{eq:CNO_discrete} corresponds to a \emph{graph convolution}\cite{li2020neural} 
defined with respect to the graph connectivity of the discretization.  
\subsection{Neural Operator from distribution to functions}\label{sec:no_distribution}
In Section~\ref{sec:no_function}, existing neural operators are designed to approximate a mapping from the input function space \(\mathcal{X}\) to the output function space \(\mathcal{Y}\). However, in our setting we instead seek to approximate the inverse operator
\(
  \mathcal{F}^{-1}:\,\mathcal{P}(H)\;\longrightarrow\;\Theta,
\)
where \(\mathcal{P}(H)\) denotes the space of probability density functions and \(\Theta\) is the target function space.  We discretize the input distribution by drawing i.i.d.\ samples
\(
  \bigl\{\rho(\cdot,t_i)\colon t_i\sim\nu_t\bigr\}_{i=1}^N.
\)
To construct a neural operator capable of modeling \(\mathcal{F}^{-1}\), Molinaro et al.~\cite{molinaro2023neural} propose that the architecture satisfy the following properties:
\begin{enumerate}
  \item \textbf{Permutation invariance.}  The model must be invariant under any permutation of the i.i.d.\ samples \(\{\rho(\cdot,t_i)\}_{i=1}^N\).
  \item \textbf{Input-size independence.}  The model must accommodate arbitrary sample size \(N\) and maintain performance regardless of \(N\).
\end{enumerate}
Among neural architectures satisfying these two properties, the most representative examples are the PointNet\cite{qi2017pointnet} and DeepSet\cite{zaheer2017deep}, originally developed for learning from point-cloud data.
These models can be viewed as a class of \emph{permutation-invariant set functions}\cite{kimura2024permutation}.
Subsequent research has focused on designing permutation neural networks that approximate such functions effectively\cite{lee2019set,liu2023flatformer,qi2017pointnet++,guo2021pct}.
In what follows, we extend the theory of permutation-invariant set functions to the infinite-dimensional setting and introduce the notion of a \emph{permutation-invariant set operator}. To make this precise, we first define function‐valued tuples as follows
\begin{definition}[Function‐valued tuple]
Let \(\mathcal{H}\) be a function space.  A \emph{tuple} (or ordered \(n\)-tuple) is an ordered list of \(n\) elements \(\{\rho_i\}_{i=1}^n\), denoted
\[
  \mathcal{S} = (\rho_1,\dots,\rho_n).
\]
\end{definition}

In practice, the empirical measure is subject to a maximal sample size \(N\) due to instrumentation limits.  The set operator
\(
  \mathcal{F}\colon \bigcup_{k=1}^N\mathcal{H}^k \longrightarrow \Theta
\)
maps a function‐valued tuple to an element of \(\Theta\).  Denote by \(\Pi_{\mathcal{S}}\) the set of all permutations of \(\mathcal{S}\).  We then define
\begin{definition}[Permutation‐invariant set operator]
A set operator
\[
  \mathcal{F}\colon \bigcup_{k=1}^N\mathcal{H}^k \longrightarrow \Theta
\]
is \emph{permutation‐invariant} if for every \(\mathcal{S}\in\bigcup_{k=1}^N\mathcal{H}^k\) and every \(\pi\in\Pi_{\mathcal{S}}\),
\begin{equation}\label{eqn:permutation_invariance}
  \mathcal{F}\bigl(\rho_{\pi(1)},\dots,\rho_{\pi(n)}\bigr)
  = 
  \mathcal{F}\bigl(\rho_1,\dots,\rho_n\bigr),
\end{equation}
where \(n=|\mathcal{S}|\).
\end{definition}

A set operator that does not satisfy~\eqref{eqn:permutation_invariance} is called \emph{permutation‐sensitive}.  Designing neural operators that satisfy permutation invariance inherently addresses input‐size independence, but poses significant challenges in efficiently aggregating across all sample orderings.  A universal framework for constructing any such operator is provided by \emph{Janossy pooling}:

\begin{definition}[Janossy pooling~\cite{murphyjanossy}]
Let
\(\mathcal{F}\colon \bigcup_{k=1}^N\mathcal{H}^k \longrightarrow \Theta\)
be any (permutation‐sensitive) set operator.  For any tuple \(\mathcal{S}\) and its permutation set \(\Pi_{\mathcal{S}}\), the \emph{Janossy pooling} of \(\mathcal{F}\) is defined by
\begin{equation}\label{eqn:janossy_pooling}
  \widehat{\mathcal{F}}(\mathcal{S})
  \;\triangleq\;
  \frac{1}{|\Pi_{\mathcal{S}}|}
  \sum_{\pi\in\Pi_{\mathcal{S}}}
    \mathcal{F}\bigl(\rho_{\pi(1)},\dots,\rho_{\pi(n)}\bigr).
\end{equation}
Furthermore, one may post‐process \(\widehat{\mathcal{F}}\) by another operator \(\mathcal{G}\), yielding
\[
  \widetilde{\mathcal{F}}(\mathcal{S})
  = \mathcal{G}\bigl(\widehat{\mathcal{F}}(\mathcal{S})\bigr),
\]
for a suitable \(\mathcal{G}\colon\mathcal{Z}\to\Theta\).
\end{definition}
While \(\widetilde{\mathcal{F}}\) is manifestly permutation-invariant, its computational cost scales as \(O(n!)\), rendering it impractical for high-dimensional tasks. The \(k\)-ary Janossy pooling offers an efficient strategy to mitigate this cost:

\begin{definition}[\(k\)\nobreakdash-ary Janossy pooling~\cite{murphyjanossy}]
  Let \(\mathcal{S}_k\) denote the collection of all \(k\)-element subsets of \(\mathcal{S}\), for some \(k<|\mathcal{S}|\).  The \(k\)-ary Janossy pooling of a set operator \(\mathcal{F}\) with post-processing operator \(\mathcal{G}\) is defined by
  \begin{equation}\label{eq_kary}
    \widetilde{\mathcal{F}}(\mathcal{S})
    \;=\;
    \mathcal{G}\!\Biggl(
      \frac{(|\mathcal{S}|-k)!}{|\mathcal{S}|}
      \sum_{\mathcal{V}\in\mathcal{S}_k}
        \mathcal{F}(\mathcal{V})
    \Biggr).
  \end{equation}
\end{definition}
This approach recovers many practical models when applied to set functions rather than set operators.  For instance, when \(k=1\), Janossy pooling for set functions is equivalent to the Deep Sets architecture, which extends to set operators as
\begin{equation}\label{eq_1ary}
      \widetilde{\mathcal{F}}(\mathcal{S})
  \;=\;
  \mathcal{G}\!\Bigl(\sum_{\rho\in\mathcal{S}}\mathcal{F}(\rho)\Bigr).
\end{equation}
where \(\mathcal{F}\) is a classical operator mapping \(\mathcal{H}\) to \(\Theta\). This formulation provides a feasible way for extending a classical operator to a set operator. Hereafter, we treat a classical operator between function spaces and a set operator with a single input element as equivalent. As an illustrative example, the \textbf{Neural Inverse Operator (NIO)} adopts this \emph{Deep Sets} framework to construct a permutation-invariant set operator.  
Specifically, it employs a DeepONet, denoted by \(\mathcal{F}_{\text{DON}}\), as the base operator \(\mathcal{F}\), and a Fourier Neural Operator (FNO), denoted by \(\mathcal{G}_{\text{FNO}}\), as the post-processing operator \(\mathcal{G}\).  
This yields a discrete approximation of a distribution-to-function operator:
\begin{equation}\label{eq_nio_integral}
\tilde{\mathcal{F}}(\mu):=\mathcal{G}_{\text{FNO}}\left(\int_{\mathcal{H}}\mathcal{F}_{\text{DON}}(\rho)\mu(d\rho) \right).
\end{equation}
where \(\mu \in \mathcal{P}(\mathcal{H})\) denotes a probability measure on \(\mathcal{H}\),\(\mathcal{Y}\) represents a latent function space, \(\mathcal{F}_{\text{DON}} : \mathcal{H} \to \mathcal{Y}\), \(\mathcal{G}_{\text{FNO}} : \mathcal{Y} \to \Theta\), and the integral is understood as a \(\mathcal{Y}\)-valued Bochner integral.

Based on equation \eqref{eq_kary}, the design of a permutation-invariant set operator must consider three aspects:
\begin{enumerate}
  \item The permutation-sensitive operator $\mathcal{F}$, which serves as a feature extractor and must capture multiscale characteristics of the input distribution; 
  \item The aggregation step—i.e., the choice of $k$—which must balance the modeling of the interaction betweeen inputs against computational cost;
  \item The post-processing module, which requires sufficient expressive capacity to recover both low- and high-frequency components of unknown functions (e.g., potential, drift, and diffusion terms) from the aggregated representation, a task that may be highly nonlinear.
\end{enumerate}

\subsection{BlinDNO}\label{sec:blindno}
Based on the analysis in Section \ref{sec:no_distribution}, we develop \textsc{BlinDNO}, an efficient architecture for approximating a permutation‐invariant set operator. We further establish the permutation invariance of the proposed architecture and provide a schematic diagram.

The inputs in our problem are high‐dimensional complex distributions, we represent, for example in two dimensions, the grid function
\(
  \rho(x) = \sum_{i=1}^{M}\sum_{j=1}^{M}\rho_{i,j}\,\phi_{i,j}(x)
  \;\longrightarrow\;
  \boldsymbol{\rho} \in \mathbb{R}^{M \times M}.
\)
 supported on a bounded domain  \(\Omega\subset\mathbb{R}^d\). Although the Fourier Neural Operator (FNO) has demonstrated effectiveness in operator learning, its global spectral convolution can be computationally expensive. On the other hand, the expressive power of the vanilla DeepONet is constrained by its network architecture in high‐dimensional, complex operator learning tasks. Under these conditions, for the computationally demanding permutation‐sensitive operator \(\mathcal{F}\), convolution‐based architectures—particularly the U‐Net—provide a favorable choice. The effectiveness of the U‐Net stems from its ability to process data on structured grids through local convolutions, a feature closely related to multigrid methods in the numerical solution of PDEs.
 \begin{figure}[!htbp]
  \centering
\includegraphics[width=\textwidth]{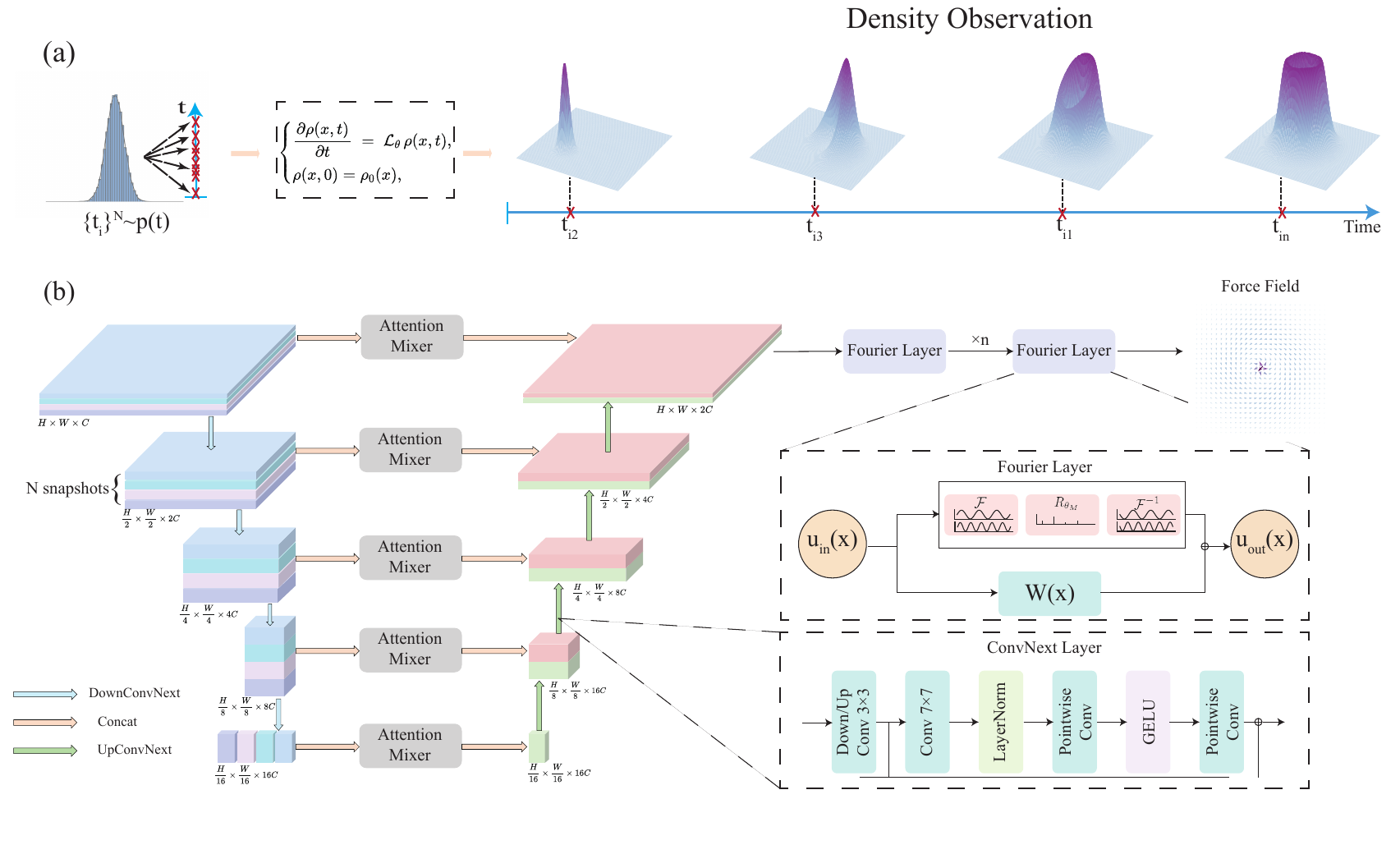}
  \caption{(a) Schematic illustration of the time-Label-Free dynamical system reconstruction problem. (b)The architecture of the \textsc{BlinDNO}.}
  \label{fig:1}
\end{figure}

\subsubsection{U-Net as a Neural Operator}\label{sec:unet_function}
We now describe the U‐Net architecture from the perspective of neural operators. The U‐Net operator is built from three principal components: (i) multi‐channel convolution, (ii) down‐ and up‐sampling, and (iii) skip connections. Let $\Omega_{h_\ell} = \{x_i^\ell\}_{i=1}^{N_\ell}\subset\Omega$ be a uniform mesh of \(\Omega\subset\mathbb{R}^d\) with spacing \(h_\ell\), and consider a nested hierarchy $ h_0 > h_1 > \cdots > h_L$. We define $V_{h_\ell}$ as the space of grid functions on $\Omega_{h_\ell}$
\[
  V_{h_\ell}(\Omega;\mathbb{R}^{d_\ell})
  = \bigl\{\,\rho\colon \Omega_{h_\ell}\to\mathbb{R}^{d_\ell}\bigr\},
\]
We define three families of linear operators between these spaces:
\begin{itemize}
  \item \emph{Downsampling Operator}: The downsampling operator $ \mathcal{D}^{\ell+1}_\ell \colon V_{h_\ell}(\Omega;\mathbb{R}^{d_\ell})
      \;\longrightarrow\;
      V_{h_{\ell+1}}(\Omega;\mathbb{R}^{d_\ell})$ maps a fine-grid function to  a coarse-grid function
    \[
      \bigl(\mathcal{D}^{\ell+1}_\ell\,\rho\bigr)(x)
      = \sum_{y\in\mathcal N_\ell(x)} \omega^\ell(y,x)\,\rho(y),
      \;x\in\Omega_{h_{\ell+1}}, 
    \]
    where \(\mathcal N_\ell(x)\) is the patch of fine‐grid nodes associated with \(x\) and \(\omega^\ell\) are learnable averaging kernels.

  \item \emph{Upsampling Operator:} The upsampling operator $ \mathcal{P}^\ell_{\ell+1}\colon V_{h_{\ell+1}}(\Omega;\mathbb{R}^{d_\ell})
      \;\longrightarrow\;
      V_{h_\ell}(\Omega;\mathbb{R}^{d_\ell})$ maps a coarse-grid function to a fine-grid function
    \[
      \bigl(\mathcal{P}^\ell_{\ell+1}\,\rho\bigr)(x)
      = \sum_{y\in\mathcal C_\ell(x)} \eta^\ell(y,x)\,\rho(y),
      \;x\in\Omega_{h_\ell},
    \]
    where \(\mathcal C_\ell(x)\) collects the coarse neighbors of $x$, and \(\eta^\ell\) are learnable interpolation kernels.

  \item \emph{Multi‐channel Convolution Operator:} The multi‐channel convolution operator $\mathcal{C}_\ell:V_{h_\ell}(\Omega;\mathbb{R}^{d_\ell})\to V_{h_\ell}(\Omega;\mathbb{R}^{d_{\ell+1}})$ acts as a localized kernel integral:
\[
  (\mathcal{C}_\ell u)(x)
  \;=\;
  \sum_{y\in B_r(x)} K_\ell(x,y)\,u(y)\,\mu_{h_\ell}(y),
  \qquad x\in\Omega_{h_\ell},
\]
where $B_r(x)$ is a radius-$r$ neighborhood, $K_\ell(x,y)\in\mathbb{R}^{d_{\ell+1}\times d_\ell}$ encodes all channel interactions and $\mu_{h_\ell}(y)$ is the discrete cell weight; $r$ is chosen small for computational efficiency .

    \item \emph{Skip Connection}: The skip operator
    \begin{equation}\label{eqn:skip}
    \begin{array}{cccl}
        \mathcal{S}_\ell :  &  V_{h_\ell}(\Omega;\mathbb{R}^{d_\ell})
      \times
      V_{h_\ell}(\Omega;\mathbb{R}^{\tilde d_\ell})&  \longrightarrow &V_{h_\ell}(\Omega;\mathbb{R}^{d_\ell+\tilde d_\ell})\\
         & \bigl(\mathcal{S}_\ell(u, \tilde u)\bigr)(x)
      & = &\bigl(u(x),\,\tilde u(x)\bigr),
      \quad
      x\in\Omega_{h_\ell} 
    \end{array}
    \end{equation}
    is the direct-sum embedding which corresponds to feature-wise concatenation and provides an identity shortcut between encoder and decoder representations at the same resolution.
\end{itemize}

In summary, the U‐Net neural operator 
\[
  \mathcal{F}_{\rm UNet}\colon V_{h_0}(\Omega;\mathbb{R}^{d_0})
  \;\longrightarrow\;
  V_{h_0}(\Omega;\mathbb{R}^{d_0})
\]
is given by the following encoder–decoder recursion.  First define the encoder outputs
\[
  e^0 = \rho,\qquad
  e^{\ell+1}
  = \bigl(\sigma\circ\mathcal{D}^{\ell+1}_\ell\circ \mathcal{C}_\ell\bigr)(e^\ell),
  \quad \ell=0,\dots,L-1.
\]
where $\sigma$ is a point-wise nonlinear activation operator(e.g. ReLU, GELU). At the bottleneck at the coarsest level, set
\[
  d^L = \mathcal{C}_L(e^L).
\]
Then propagate through the decoder via skip‐connections:
\[
  d^\ell
  = \bigl(\sigma\circ\mathcal{C}_\ell\circ \mathcal{S}_\ell\circ \bigl[e^\ell,\;
    \mathcal{P}^\ell_{\ell+1}(d^{\ell+1})\bigr]\bigr),
  \quad \ell=L-1,\dots,0,
\]
where \(\bigl[e^\ell,\;\mathcal{P}^\ell_{\ell+1}(d^{\ell+1})\bigr]\) denotes the pair of encoder and upsampled decoder features at level \(\ell\).  Finally, the output of the UNet operator is
\begin{equation}\label{eqn:UNet}
    \mathcal{F}_{\rm UNet}(\rho) = d^0.
\end{equation}
\subsubsection{Permutation-invariant UNet}
We now turn to the architecture of \textsc{BlinDNO}. Although the vanilla U-Net is computationally efficient, its simple structure often underperforms global convolutional networks such as the Fourier Neural Operator when tackling high-dimensional, complex problems. To enhance the representational capacity of the U-Net, we apply ConvNeXt layers to the outputs of both the encoder and the decoder (see Fig.~1).

The ConvNeXt module applies a $7\times7$ depthwise convolution for larger receptive field, followed by layer normalization, a two‐layer pointwise MLP with GELU nonlinearity, and a residual connection to preserve gradient flow:
\begin{equation}
\mathcal{C}_{\text{CN}}(\rho) \;\triangleq\; \rho\;+\;\mathcal{C}_{1\times1,2}\circ\;\mathrm{GELU}\circ\mathrm{LN}\circ \mathcal{C}_{1\times1,1}\,\circ\mathcal{DC}_{7\times7}(\rho)\,,
\end{equation}
where $\mathcal{DC}_{7\times7}$ is a depthwise convolution with a $7\times7$ kernel, $\mathcal{C}_{1\times1,1}$ and $\mathcal{C}_{1\times1,2}$ are pointwise convolutions, LN denotes layer normalization, and GELU is the Gaussian Error Linear Unit.

To establish a permutation‐invariant set operator, we represent the inputs
\[
(\rho_1,\ldots,\rho_N) = \bigl(\rho(\cdot,t_1),\ldots,\rho(\cdot,t_N)\bigr)
\]
as an $N$‐tuple. For convenience, define
\[
\mathcal{E}_{\ell+1}(\cdot)
\;=\;
\prod_{i=1}^{\ell}
\!\Bigl(\mathcal{C}_{\mathrm{CN}}\circ\sigma\circ\mathcal{D}^{\ell+1}_{\ell}\circ \mathcal{C}_{\ell}\Bigr)(\cdot)
\]
to be the overall encoder for the $(\ell+1)$-th layer. Applying this shared‐parameter encoder to each input yields the output tuple at layer $\ell$:
\begin{equation}\label{eqn:input_enc}
e_{\ell+1}
\;=\;
(z_1,\ldots,z_N)
\;=\;
\bigl(\mathcal{E}_{\ell+1}\rho_{1},\ldots,\mathcal{E}_{\ell+1}\rho_{N}\bigr).
\end{equation}
Since the inputs are i.i.d.\ realizations from a common distribution, the mutual information between the elements of $\mathcal{S}_{\ell+1}$ can facilitate modeling of the input distribution and thereby improve predictive accuracy. To this end, we employ multi‐head self‐attention, which is permutation‐equivariant, to extract multi‐scale features with high mutual information. We first consider single‐head attention mechanism and define the key, query, and value operators at layer $\ell$ as
\begin{equation}
\begin{aligned}
\mathcal{K}_{\ell+1}\colon &\; V_{h_{\ell+1}}(\Omega;\mathbb{R}^{d_{\ell+1}})
\;\to\;
V_{h_{\ell+1}}(\Omega;\mathbb{R}^{d_{k}}),\\
\mathcal{Q}_{\ell+1}\colon &\; V_{h_{\ell+1}}(\Omega;\mathbb{R}^{d_{\ell+1}})
\;\to\;
V_{h_{\ell+1}}(\Omega;\mathbb{R}^{d_{q}}),\\
\mathcal{V}_{\ell+1}\colon &\; V_{h_{\ell+1}}(\Omega;\mathbb{R}^{d_{\ell+1}})
\;\to\;
V_{h_{\ell+1}}(\Omega;\mathbb{R}^{d_{v}}),
\end{aligned}
\end{equation}
each being a linear map that lifts the input function to the appropriate dimension. Let $\{k_j\}_{j=1}^N$, $\{q_j\}_{j=1}^N$, and $\{v_j\}_{j=1}^N$ denote the input key, query, and value functions in $V_{h_{\ell+1}}(\Omega;\mathbb{R}^{d_{\ell+1}})$. Assuming $d_{k}=d_{q}$, we denote
\[
k^j = \mathcal{K}_{\ell+1}[k_j],\quad
q^j = \mathcal{Q}_{\ell+1}[q_j],\quad
v^j = \mathcal{V}_{\ell+1}[v_j],
\]
the embedded key, query, and value functions. The output of the single‐head attention operator for token $j$ is then
\begin{equation}\label{eqn:output_sa}
o^j
\;=\;
\mathrm{Softmax}\!\Bigl(\tfrac{1}{\tau}
\bigl[\langle q^j,k^1\rangle,\ldots,\langle q^j,k^N\rangle\bigr]\Bigr)
\;\bigl[v^1,\ldots,v^N\bigr]^\top,
\end{equation}
where $\tau>0$ is the temperature hyperparameter. The single-head self-attention operator corresponds to the special case $k_j=q_j=v_j$.

The multi-head attention applies the single-head attention described above separately for multiple head $h\in\{1\ldots,H\}$ and concatenate the output together
\begin{equation}\label{eqn:output_attn}
\mathbf{o}^j = \text{Concat}(o^{j,1},\ldots,o^{j,H})     
\end{equation}
where $o^{j,i}$ is the $j$-th output of the $i$-th head. For convenience, we denote the multi-head self-attention operator and cross-attention operator as $\text{SelfAttn}(\cdot)$ and $\text{Attn}(\cdot,\cdot,\cdot)$, respectively. 

Finally, to obtain a permutation‐invariant, set‐level feature, we average over the output tokens of multi-head self-attention operator
\begin{equation}\label{eqn:out_enc_feqture}
    \tilde{e}_{\ell+1} = \frac{1}{N}\sum_{j=1}^N \mathbf{o}^j = \text{Avg}\circ\text{SelfAttn}\circ \prod_{i=1}^{\ell}
\!\Bigl(\mathcal{C}_{\mathrm{CN}}\circ\sigma\circ\mathcal{D}^{\ell+1}_{\ell}\circ \mathcal{C}_{\ell}\Bigr)((\rho_1,\ldots,\rho_N)).
\end{equation}
These encoder features, extracted at multiple scales, are injected into the decoder via skip connections (cf.\ \eqref{eqn:skip}), and the final output is produced by successive upsampling layers:
\begin{equation}\label{eqn:output_unet}
  d^\ell
  = \bigl(\sigma\circ\mathcal{C}_\ell\circ \mathcal{S}_\ell\circ \bigl[\tilde{e}^{\ell+1},\;
    \mathcal{P}^\ell_{\ell+1}(d^{\ell+1})\bigr]\bigr),
  \quad \ell=L-1,\dots,0,
\end{equation}

We denote the above neural operator as $\mathcal{F}_{BlinDNO}$. The following proposition establishes its invariance under input permutations.

\begin{proposition}
    The modified UNet operator $\mathcal{F}_{BlinDNO}\colon \bigcup_{k=1}^N\mathcal{H}^k \longrightarrow \Theta$ is a permutation invariant set operator.
\end{proposition}

\begin{proof}
It suffices to show that the encoder $\eqref{eqn:out_enc_feqture}$ of each layer is permutation invariant. 
For clarity, we consider the single-head attention version.
 Let $e_{\ell+1} = (z_1,\ldots,z_N)$ be  the output tuple at layer $l$. We then introduce a fixed tuple of functions $\mathcal{T} =(y_1,\ldots,y_N)$ as the input of key operator at layer~$\ell+1$. The attention with keys $k^i=\mathcal{K}_{\ell+1}[y_i]$, queries  $q^i=\mathcal{Q}_{\ell+1}[z_i]$ and values $v^i=\mathcal{V}_{\ell+1}[z_i]$, followed by averaging, reads
\begin{equation}\label{eq:attn_avg_general}
\mathrm{Avg}\circ\mathrm{Attn}(e_{\ell+1},\mathcal{T},e_{\ell+1})
=\frac{1}{N}\sum_{j=1}^N 
\frac{\displaystyle\sum_{i=1}^N 
    \exp\!\left(\frac{\langle q^j,k^i\rangle}{\tau}\right) v^i}
     {\displaystyle\sum_{t=1}^N 
    \exp\!\left(\frac{\langle q^j,k^t\rangle}{\tau}\right)}.
\end{equation}

Define the pairwise kernels
\[
\psi(x_j,x_i)
:=\exp\!\Bigl(\tfrac{\langle \mathcal{Q}_{\ell+1}[x_j],\mathcal{K}_{\ell+1}[x_i]\rangle}{\tau}\Bigr)
    \,\mathcal{V}_{\ell+1}[x_i],
\qquad
\chi(x_j,x_i)
:=\exp\!\Bigl(\tfrac{\langle \mathcal{Q}_{\ell+1}[x_j],\mathcal{K}_{\ell+1}[x_i]\rangle}{\tau}\Bigr).
\]

We now show that the mapping in~\eqref{eq:attn_avg_general} can be written as a {\it 2-ary Janossy pooling}.
Consider the function
\[
H(z_j,z_i;\mathcal{T})
\;:=\;
\frac{N\,\psi(z_j,z_i)}{\displaystyle\sum_{t=1}^N\chi(z_j,y_t)}.
\]
Observe that the denominator is symmetric in~$\mathcal{T}$, so for any permutation $\pi$ of $\{1,\dots,N\}$,
\[
H(z_{\pi(j)},z_{\pi(i)};\pi\mathcal{T})=H(z_j,z_i;\mathcal{T}).
\]
Therefore, the aggregation with {\it 2-ary Janossy pooling} gives
\begin{equation}\label{eq:janossy_form}
\begin{array}{cll}
   \mathcal{J}_2(e_{\ell+1})& := & \frac{1}{N^2}\sum_{j=1}^N\sum_{i=1}^N H(z_j,z_i;\mathcal{T}) \\
     &= &\frac{1}{N}\sum_{j=1}^N
  \frac{\sum_{i=1}^N \psi(z_j,z_i)}{\sum_{i=1}^N \chi(z_j,z_i)}
=\mathrm{Avg}\circ\mathrm{Attn}(e_{\ell+1},\mathcal{T},e_{\ell+1})
\end{array}
\end{equation}
Hence the attention–average operator coincides exactly with a 2-ary Janossy pooling.

Finally, by setting $\mathcal{T}$ to the encoder output tuple 
$e_{\ell+1}=(z_1,\ldots,z_N)$ yields
\[
\tilde e_{\ell+1}
=\mathrm{Avg}\circ\mathrm{Attn}(e_{\ell+1},e_{\ell+1},e_{\ell+1})
=\mathcal{J}_2(e_{\ell+1}),
\]
establishing the permutation invariance of the encoder layer.
Since the decoder in~\eqref{eqn:output_unet} consists solely of pointwise convolutions, upsampling, and skip connections operating on permutation-invariant features, it introduces no order dependence. 
Therefore, the overall operator 
$\mathcal{F}_{\mathrm{BlinDNO}}$ 
is permutation invariant.
\end{proof}

From a distributional perspective, the proposed attention–then–average mechanism admits an integral representation analogous to~\eqref{eq_nio_integral}. 
Specifically, let $\{z_1,\ldots,z_{N}\}$ in the output tuple $e_{\ell+1}$ at layer~$\ell$ be i.i.d.~samples drawn from a probability distribution 
$\mu$ over the function space $V_{h_{\ell+1}}(\Omega;\mathbb{R}^{d_{\ell+1}})$, i.e., 
$z_i\sim \mu\in \mathcal{P}(V_{h_{\ell+1}}(\Omega;\mathbb{R}^{d_{\ell+1}}))$. 
Then, the empirical output of the attention and averaging operation in~\eqref{eqn:out_enc_feqture} can be written as
\begin{equation}
\begin{aligned}
\text{Avg}\circ\text{Attn}(e_{\ell+1},e_{\ell+1},e_{\ell+1})
&= \frac{1}{N}\sum_{j=1}^N 
   \frac{\sum_{i=1}^N \exp\!\left(\tfrac{\langle q^j,k^i\rangle}{\tau}\right)v^i}
        {\sum_{t=1}^N \exp\!\left(\tfrac{\langle q^j,k^t\rangle}{\tau}\right)} \\[0.8ex]
&\xrightarrow[N\to\infty]{} 
   \int 
   \frac{\displaystyle\int 
      \exp\!\left(\tfrac{\langle \mathcal{Q}_{\ell+1}[z],\mathcal{K}_{\ell+1}[z']\rangle}{\tau}\right)
      \mathcal{V}_{\ell+1}[z']\,\mu(dz')}
      {\displaystyle\int 
      \exp\!\left(\tfrac{\langle \mathcal{Q}_{\ell+1}[z],\mathcal{K}_{\ell+1}[s]\rangle}{\tau}\right)
      \mu(ds)}
   \,\mu(dz),
\end{aligned}
\label{eq:attn_integral}
\end{equation}
where the convergence holds in the Bochner sense under standard integrability and boundedness assumptions. 
We denote the right-hand side of~\eqref{eq:attn_integral} by $\mathcal{A}_{\ell+1}[\mu]$. Equation~\eqref{eq:attn_integral} thus defines a nonlinear double integral operator on the probability measure~$\mu$, 
in which the inner integral corresponds to an exponentially tilted expectation with respect to $\mu$, 
and the outer integral averages over all query samples.

Furthermore, as the temperature parameter $\tau \to \infty$, the exponential kernel becomes asymptotically uniform, i.e.,
\[
\exp\!\left(\tfrac{\langle \mathcal{Q}_{\ell+1}[z],\mathcal{K}_{\ell+1}[z']\rangle}{\tau}\right)
= 1 + \mathcal{O}(\tfrac{1}{\tau}),
\]
so that the normalized attention weights converge to unity. 
In this limit, the operator~$\mathcal{A}_{\ell+1}[\mu]$ reduces to the single-integral form
\begin{equation}
\mathcal{A}_{\ell+1}[\mu]
\;\xrightarrow[\tau\to\infty]{}\;
\int \mathcal{V}_{\ell+1}[z]\,\mu(dz),
\label{eq:attn_to_DS}
\end{equation}
which coincides with the Deep Sets (or $k=1$ Janossy pooling) representation~\eqref{eq_nio_integral}. 
Hence, the proposed attention–average mechanism can be interpreted as a continuous generalization of the Deep Sets formulation, 
where finite~$\tau$ introduces an exponentially weighted coupling between samples, thereby capturing higher-order statistical interactions within the input distribution~$\mu$.

\section{Numerical Experiments}\label{sec:result}
In this section, we present a series of numerical experiments to comprehensively evaluate the performance of the proposed method under diverse settings. 
We begin with one-dimensional demonstrations: 
(i) recovering stochastic dynamics governed by a stochastic differential equation (SDE) with unknown drift and diffusion coefficients, and 
(ii) inferring both linear and nonlinear Schr\"{o}dinger equations with unknown potential and nonlinear terms. 
Subsequently, we test our method on two-dimensional SDEs featuring non-uniform diffusion and non-conservative force fields. 
Finally, we apply \textsc{BlinDNO} to reconstruct realistic protein-folding trajectories, explicitly assessing its performance under varying observation-time distributions.

For all experiments, we benchmark against two baselines: the \textbf{Neural Inverse Operator (NIO)} and an augmented variant, \textbf{FNO-NIO}, in which the permutation-sensitive operator of NIO is replaced by Fourier Neural Operators (FNOs) to enhance representational capacity. In all of the experiments, we record $100$ density snapshots $\{p(\cdot,t_i)\}_{i=1}^{100}$ at randomly selected times $t_i\sim\nu_t$ as the model input.
To assess inversion accuracy, we report the relative $\ell_{2}$ error of each inferred quantity,
\begin{equation}
\boldsymbol{E}_{\theta}
=\frac{\lVert \theta-\theta^{\star}\rVert_{2}}
       {\lVert \theta^{\star}\rVert_{2}},
\end{equation}
where the parameter $\theta$ represents the drift function $\boldsymbol{\mu}$, diffusion matrix $D$, potential $V(x)$, or other unknown quantities of interest.

In addition, we evaluate the reconstruction quality of the inferred density $\rho_{\theta}(x,t)$ by computing its \emph{time-averaged relative $L^2$ error} against the reference density $\rho_{\theta^{\star}}(x,t)$, namely,
\begin{equation}
\boldsymbol{E}_{\boldsymbol{\rho}}
=\frac{1}{T}\int_{0}^T
  \frac{\lVert \rho_{\theta}(\cdot,t)
       -\rho_{\theta^{\star}}(\cdot,t)\rVert_{2}}{\lVert \rho_{\theta^{\star}}(\cdot,t)\rVert_{2}}dt
\end{equation}
This metric quantifies the average relative discrepancy between the inferred and ground-truth density evolutions over time.

All computations were performed on a workstation equipped with four NVIDIA RTX-4090 GPUs and 128~GB of system memory. 
Comprehensive details regarding dataset generation and training protocols are provided in Appendix~SM1.

\subsection{Example 1: 1D Fokker–Planck Equation}
\label{subsec:1d_fpe_mog}

We begin with an inverse reconstruction problem for a one-dimensional It\^{o} diffusion process whose drift term is derived from a mixture-of-Gaussians potential and whose diffusion coefficient is an unknown constant. The corresponding probability density $\rho(x,t)$ evolves according to the Fokker–Planck equation
\begin{equation}\label{eq:1d_sde_fpe_compact}
\partial_t \rho(x,t) = -\partial_x\!\bigl(\mu(x)\rho(x,t)\bigr) + D\,\partial_{xx}\rho(x,t), \qquad (x,t)\in \mathbb{R}\times[0,1],
\end{equation}
where the drift is given by 
\begin{equation}\label{eq:mog_drift}
\mu(x) = -\partial_x U(x), \quad 
U(x) = \sum_{i=1}^{3} A_i \exp\!\Bigl(-\frac{(x-c_i)^2}{2\sigma_i^2}\Bigr),
\end{equation}
where the potential parameters are randomly sampled as
$A_i\sim \mathrm{Unif}[1,2]$, 
$c_i\sim \mathrm{Unif}\!\bigl[\tfrac{5}{16},\tfrac{11}{16}\bigr]$, 
$\sigma_i\sim \mathrm{Unif}[0.025,0.1]$, and
$D\sim \mathrm{Unif}[1,2]$.

The computational domain is restricted to $\Omega = [0,1]$ with absorbing boundary conditions. Temporal observation times are independently sampled from the uniform distribution $\nu_t = \mathrm{Unif}[0,1]$, and the spatial discretization is fixed at $\Delta x = 1/80$. For each realization of $(\mu, D)$, we numerically evolve \eqref{eq:1d_sde_fpe_compact} on a refined temporal mesh to ensure stability and accuracy, employing a thermodynamically consistent finite-difference solver \cite{holubec2019physically}. Notably, as $D$ is constant, its estimate is taken as the spatial mean of the network’s output.

Figure~\ref{fig:2} and Table~\ref{tab:1} summarize the reconstruction results obtained by \textsc{BlinDNO} and several baseline methods. As shown in Fig.~\ref{fig:2}(a–c), \textsc{BlinDNO} achieves substantially more accurate recovery of both the drift and diffusion, particularly in regions near potential saddles. Furthermore, Figs.~\ref{fig:2}(d–h) demonstrate that the dynamics reconstructed by our method produce solutions closely aligned with the true system evolution, whereas baseline approaches exhibit pronounced deviations, especially in transient regimes.

 \begin{figure}[!htbp]
  \centering
\includegraphics[width=\textwidth]{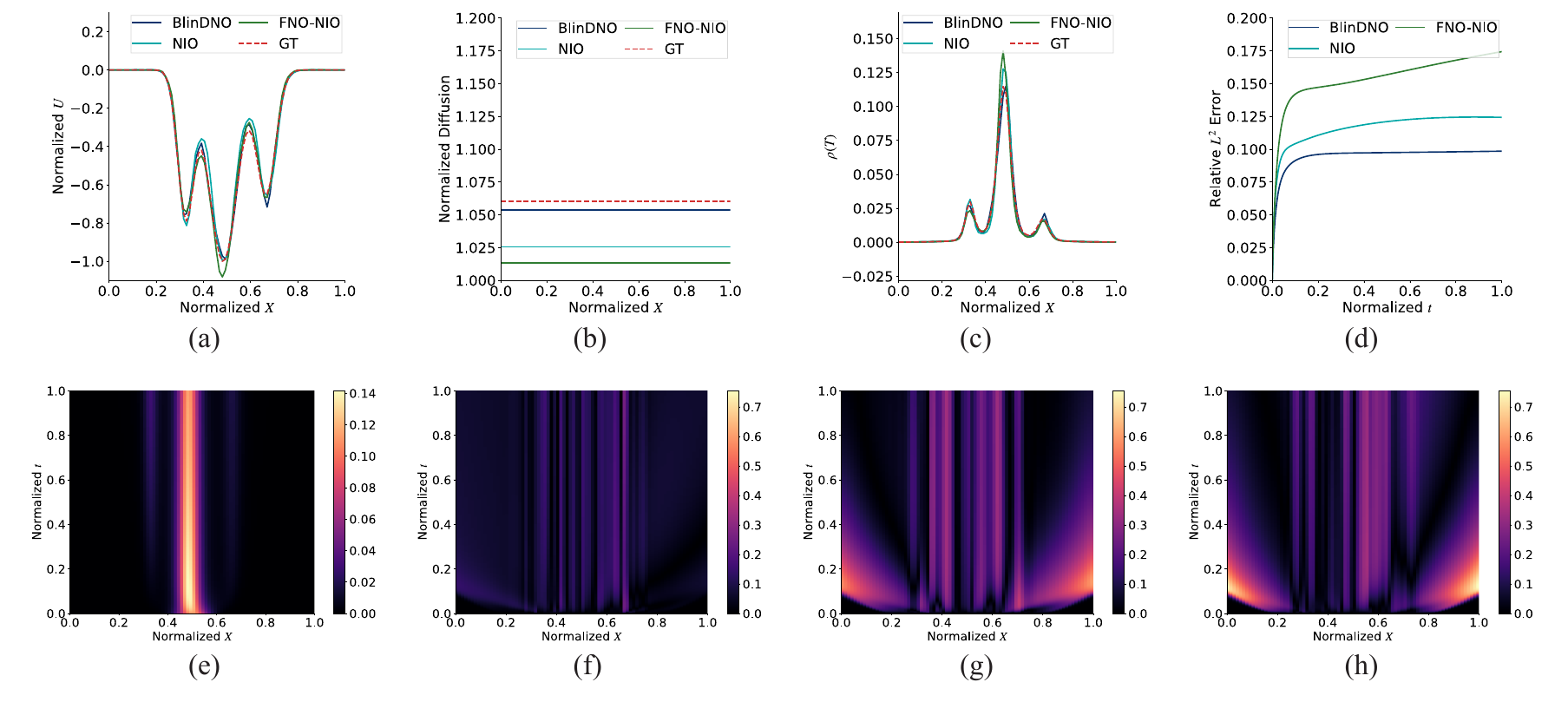}
  \caption{Inversion results for the one-dimensional Fokker–Planck equation (FPE) problem. 
(a) Reconstructed potential profile, $U(x)$. 
(b) Reconstructed diffusion term. 
(c) Comparison between the density $\rho(T)$ obtained from simulations of the reconstructed dynamical system and the ground truth. 
(d) Relative $L_{2}$ error at each time step between the simulated density and the ground truth. 
(e) Ground-truth density distribution, $\rho_{\mathrm{GT}}(x,t)$. 
(f)–(h) Pointwise errors in the simulated density function $\rho(x,t)$, obtained using the reconstructed potential and diffusion term with \textsc{BlinDNO}, NIO, and FNO–NIO.
}
  \label{fig:2}
\end{figure}

\subsection{Example 2: 1D Linear/Nonlinear Schrödinger Equation} \label{subsec:1d_schrodinger}
We next consider the inverse potential reconstruction problem for both the one-dimensional linear Schrödinger equation and  the Gross–Pitaevskii equation (GPE). In either case, the goal is to recover the unknown external potential $V(x)$ from unordered measurements of the probability density $\rho(x,t)=|\Psi(x,t)|^2$. The linear dynamics are governed by
\begin{equation}\label{eq:linear_schrodinger}
i\hbar\,\partial_t\Psi(x,t)
= \left[-\frac{\hbar^2}{2m}\nabla^2 + V(x)\right]\Psi(x,t),
\end{equation}
while the GPE augments \eqref{eq:linear_schrodinger} with nonlinear interaction terms,
\begin{equation}\label{eq:gpe}
i\hbar\,\partial_t\Psi(x,t)
= \left[-\frac{\hbar^2}{2m}\nabla^2 + V(x) + \beta|\Psi(x,t)|^2 + \kappa|\Psi(x,t)|^4\right]\Psi(x,t),
\end{equation}
where, in our experiments, we fix $\beta = \kappa = 2$. The computational domain is taken as $\Omega = [-10,10]$ with $N_x = 128$ uniformly spaced grid points, and the initial condition is prescribed by
\[
\Psi(x,0)=\frac{\sin(x)}{\cosh(x)}.
\]

The external potential combines harmonic and periodic components and is parameterized as
\begin{equation}\label{eq:V_form}
V(x) = a\,(x-x_0)^2 + b\,\cos\!\left(c\,(x-x_0)\right)^2,
\end{equation}
where the potential parameters are randomly sampled as $a\sim\mathrm{Unif}[0.1,0.3], 
b\sim\mathrm{Unif}[0.5,2],
c\sim\mathrm{Unif}[0.5,2],
x_0\sim\mathrm{Unif}[-3,3].$

For each realization of $V(x)$, we solve either \eqref{eq:linear_schrodinger} or \eqref{eq:gpe} up to $t_{\mathrm{final}}=5.0$ using a maximum time step of $\Delta t_{\max}=0.005$ and a second-order Strang splitting scheme to ensure stability and accuracy. Following this procedure, we generate $5000$ training samples and $1000$ testing samples.

As shown in Figs.~\ref{fig:3}(b) and \ref{fig:3}(f), the NIO model fails to capture the periodic component of the potential and primarily recovers only its harmonic structure. In contrast, both \textsc{BlinDNO} and \textsc{FNO--NIO} successfully reconstruct the harmonic and periodic components, with \textsc{BlinDNO} achieving noticeably higher accuracy. This improvement is further reflected in the dynamical simulations presented in Figs.~\ref{fig:3}(c)–(d) and \ref{fig:3}(g)–(h), where the trajectories generated using the reconstructed potentials show that \textsc{BlinDNO} yields dynamics that remain substantially closer to the ground truth. Quantitative comparisons are summarized in Table~\ref{tab:1}.

 \begin{figure}[!htbp]
  \centering
\includegraphics[width=\textwidth]{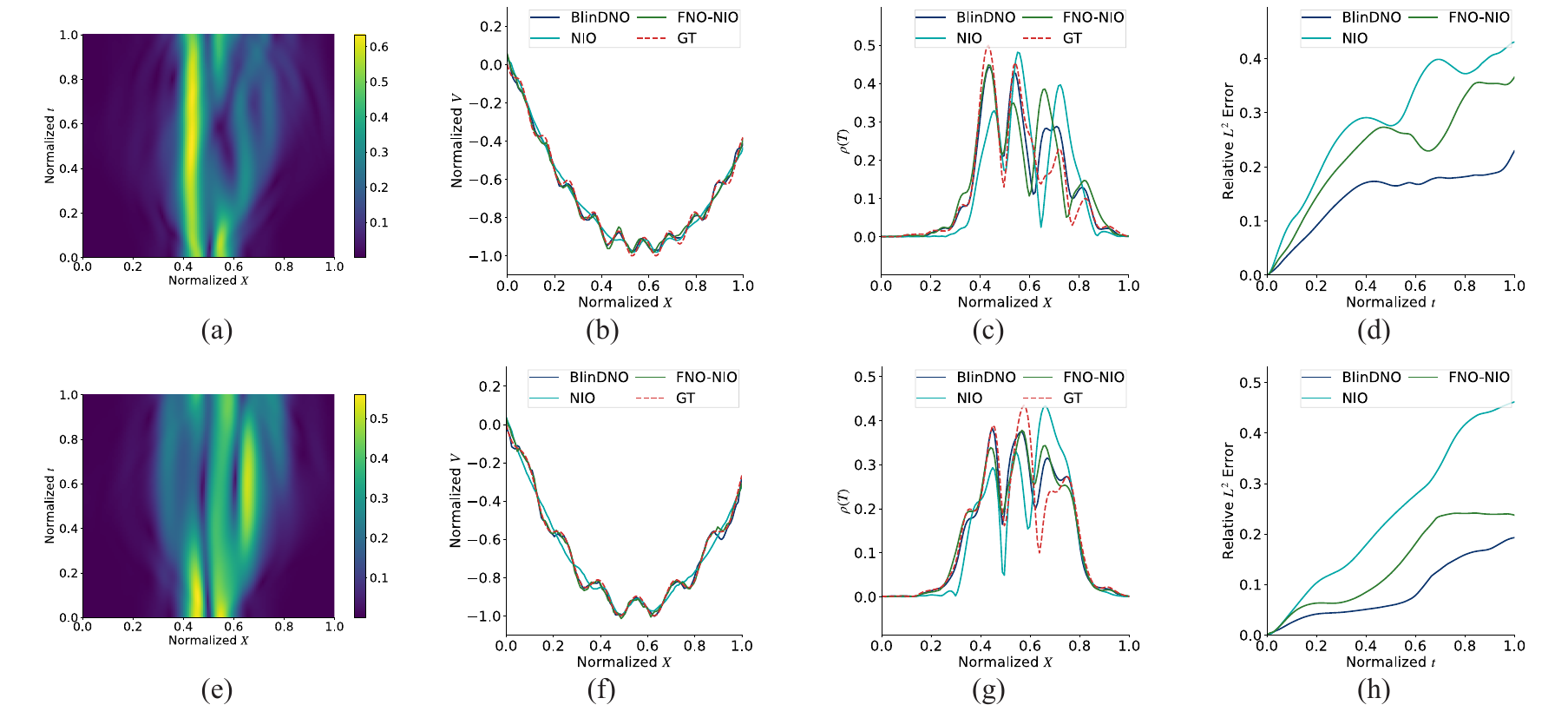}
  \caption{Inversion results for the 1D linear/nonlinear Schrodinger equation problem. 
(a) Ground-truth density distribution of the 1D linear Schrodinger equation, $\rho_{\mathrm{GT}}(x,t)$. 
(b) Reconstructed potential of the 1D linear Schrodinger equation, $V(x)$. 
(c) Comparison between the density $\rho(T)$ obtained from simulations of the reconstructed dynamical system and the ground truth. 
(d) Relative $L_{2}$ error at each time step between the simulated density and the ground truth. 
(a) Ground-truth density distribution of the 1D Gross-Pitaevskii equation, $\rho_{\mathrm{GT}}(x,t)$. 
(b) Reconstructed potential of the 1D 1D Gross-Pitaevskii equation, $V(x)$. 
(c) Comparison between the density $\rho(T)$ obtained from simulations of the reconstructed dynamical system and the ground truth. 
(d) Relative $L_{2}$ error at each time step between the simulated density and the ground truth.
}
  \label{fig:3}
\end{figure}
\subsection{Example 3: 2D Fokker Planck Equation with Non-uniform Diffusion Term}
\label{subsec:2d_fpe_nonuniform}
We next assess the performance of our framework on a 2D inverse problem involving the joint reconstruction of the drift field and a non-uniform diffusion coefficient. The dynamics are governed by an SDE whose drift is derived from a mixture-of-Gaussian potential, while the diffusion tensor is diagonal with nonconstant entries $D_{xx}(\boldsymbol{x})=D_{yy}(\boldsymbol{x})\equiv D(\boldsymbol{x})$. The corresponding probability density $p(\boldsymbol{x},t)$ evolves according to
\begin{equation}\label{eq:2d_fpe_nonuniform}
\partial_t \rho
= -\nabla\cdot(\boldsymbol{\mu}\,\rho)
  + \partial_{x_1x_1}\!\bigl(D\,\rho\bigr)
  + \partial_{x_2x_2}\!\bigl(D\,\rho\bigr),
\qquad (\boldsymbol{x},t)\in\mathbb{R}^2\times[0,1].
\end{equation}

The drift field is specified as
\begin{equation}\label{eq:2d_drift}
\boldsymbol{\mu}(\boldsymbol{x})=\nabla U(\boldsymbol{x}),
\qquad
U(\boldsymbol{x})
= \sum_{i=1}^{3}
A_i
\exp\!\Bigl(-\frac{\|\boldsymbol{x}-\boldsymbol{c}_i\|^2}{2\sigma_i^2}\Bigr),
\end{equation}
where the parameters are randomly sampled as  
\(
A_i\sim\mathrm{Unif}[1,2],
\boldsymbol{c}_i\sim\mathrm{Unif}\!\Bigl[\tfrac{1}{3},\tfrac{2}{3}\Bigr]^2,
\sigma_i\sim\mathrm{Unif}\!\Bigl[\tfrac{1}{30},\tfrac{4}{30}\Bigr].
\)
The diffusion coefficient is modeled by a harmonic function of the form
\begin{equation}\label{eq:2d_diffusion}
D(\boldsymbol{x}) = 1 + \alpha\bigl[(x_1-p_1)^2 + (x_2-p_2)^2\bigr],
\end{equation}
with $\alpha\sim\mathrm{Unif}[0,2]$ and $(p_1,p_2)\sim\mathrm{Unif}\!\bigl[\tfrac{1}{3},\tfrac{2}{3}\bigr]^2$.

The computational domain is restricted to $\Omega = [0,1]^2$ with absorbing boundary conditions. The spatial resolution is $\Delta x = 1/60$ and temporal observation times are independently sampled from $\nu_t=\mathrm{Unif}[0,1]$. For each realization of $(\boldsymbol{\mu},D)$, we solve \eqref{eq:2d_fpe_nonuniform} using the same numerical scheme described in Sec.~\ref{subsec:1d_fpe_mog}, advancing the solution on a refined temporal mesh to ensure stability and accuracy.  Following this procedure, we obtain a dataset comprising $3200$ training samples and $800$ testing samples.

As illustrated in Figs.~\ref{fig:4}(c)–(j), and consistent with the one-dimensional experiments, the NIO model fails to accurately reconstruct the underlying potential. In contrast, both \textsc{FNO--NIO} and \textsc{BlinDNO} recover the unknown potential with significantly higher fidelity, with \textsc{BlinDNO} achieving the most accurate reconstructions among the three. For the diffusion term, all models deliver comparably reliable reconstructions. The recovered probabilistic dynamics shown in Fig.~\ref{fig:4}(b), together with the quantitative error metrics in Table~\ref{tab:1}, further demonstrate the superior performance of \textsc{BlinDNO}.

 \begin{figure}[!htbp]
  \centering
\includegraphics[width=\textwidth]{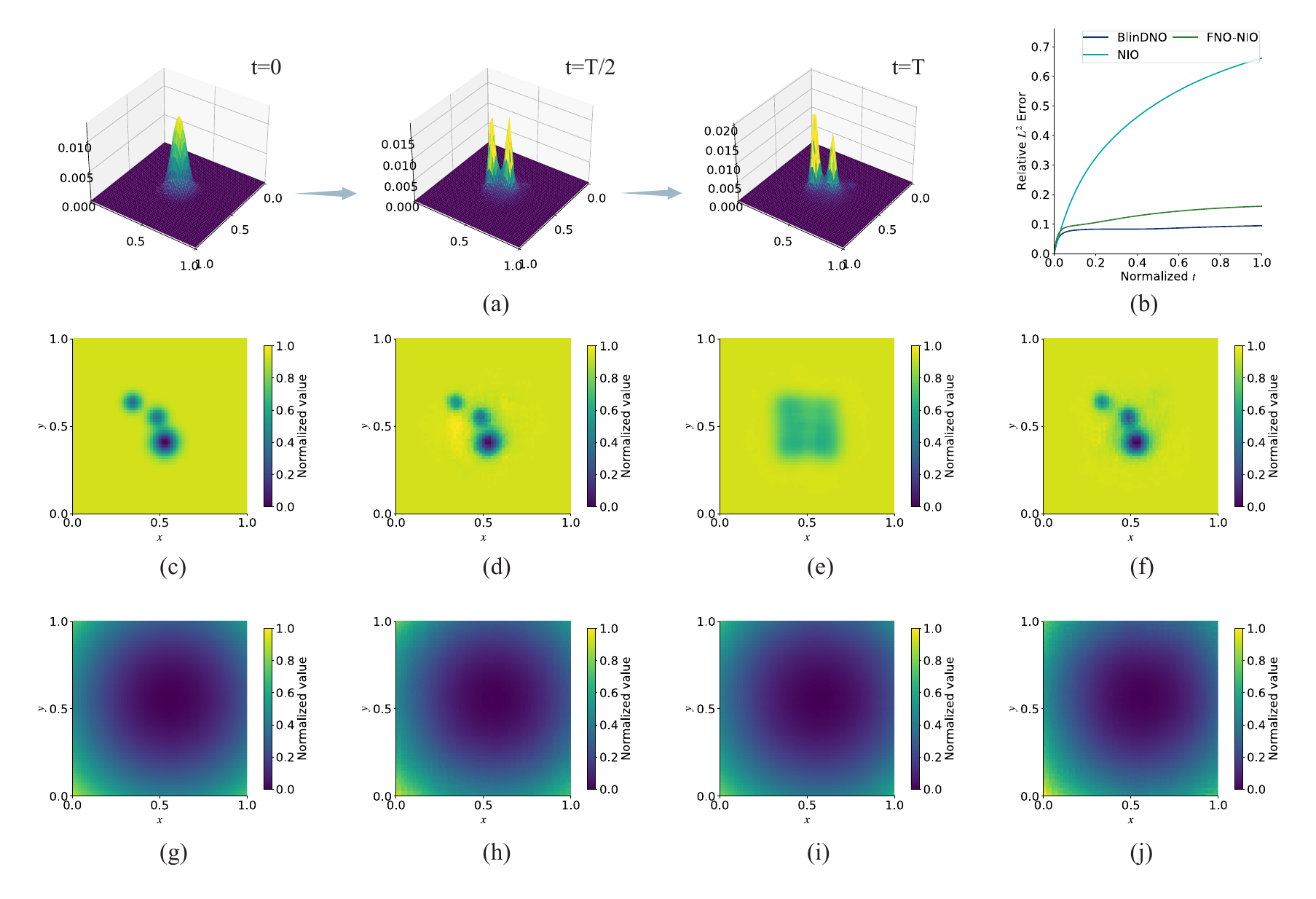}
  \caption{Inversion results for the 2D Fokker–Planck equation with a non-uniform diffusion term. 
(a) Temporal evolution of the ground-truth density, $\rho_{\mathrm{GT}}(x,t)$. 
(b) Relative $L_{2}$ error at each time step between the simulated density and the ground truth. 
(c) Ground-truth potential, $U(x)$. 
(d)–(f) Reconstructed potential profiles, $U(x)$, obtained using \textsc{BlinDNO}, NIO, and FNO–NIO. 
(g) Ground-truth diffusion term, $D(x)$. 
(h)–(j) Reconstructed diffusion terms, $D(x)$, obtained using \textsc{BlinDNO}, NIO, and FNO–NIO.
}
  \label{fig:4}
\end{figure}
\subsection{Example 4: 2D Fokker Planck Equation with Non-conservative Force Field}
\label{subsec:2d_fpe_noncons}
In many practical scenarios the drift term is not derived from a potential, leading to intrinsically non--equilibrium dynamics. Here we investigate a setting with a non--conservative force field exhibiting both rotational structure and radial dissipation.  The probability density $p(\boldsymbol{x},t)$ evolves according to
\begin{equation}\label{eq:2d_fpe_noncons}
\partial_t p
= -\nabla\cdot(\boldsymbol{\mu}\,p) + D\,\Delta p,
\qquad (\boldsymbol{x},t)\in[0,1]^2\times[0,1],
\end{equation}
where $D>0$ is a known constant. The drift field $\boldsymbol{\mu}(\boldsymbol{x})=(F_x,F_y)$ is parameterized in polar coordinates $(r,\phi)$ through tangential and radial components $F_\phi$ and $F_r$, respectively:
\begin{equation}\label{eq:2d_forcefield}
\begin{aligned}
F_\phi(r) &= \gamma_\phi\,\frac{r}{L}\,\exp\!\left(-\frac{b\,r}{L}\right), \\
F_r(r)    &= \gamma_r\left(1-\frac{r}{L}\right)\exp\!\left(-\frac{d\,r}{L}\right),
\end{aligned}
\end{equation}
with $r=\sqrt{x_1^2+x_2^2}$ and $\phi=\arctan2(x_2,x_1)$. The Cartesian components are obtained via the standard polar-to-Cartesian transformation:
\begin{equation}\label{eq:force_cartesian}
\begin{aligned}
F_x &= -\sin\phi\,F_\phi(r) + \cos\phi\,F_r(r), \\
F_y &= \phantom{-}\cos\phi\,F_\phi(r) + \sin\phi\,F_r(r).
\end{aligned}
\end{equation}
The parameters are sampled as
\(
L\sim\mathrm{Unif}[0.25,0.75],
\gamma_\phi,\gamma_r\sim\mathrm{Unif}[0.5,2],
b,d\sim\mathrm{Unif}[0.5,2],
\)
where $\gamma_\phi$ and $\gamma_r$ control the strength of the tangential and radial components, while $b$ and $d$ determine their exponential decay rates.  We adopt a setting analogous to that of the previous example, but employ a finer spatial discretization with grid spacing $\Delta x = 1/80$.

In this setting, all three models are able to reconstruct the underlying force field successfully. As illustrated by the example in Fig.~\ref{fig:5}, the \textsc{FNO--NIO} model tends to underestimate the overall magnitude of the force field, whereas \textsc{BlinDNO} provides more accurate predictions near the center compared to \textsc{NIO}. Consequently, the trajectories generated using the \textsc{BlinDNO}-recovered field remain closer to the true dynamics. The statistical results reported in Table~\ref{tab:1} further corroborate the superior performance of \textsc{BlinDNO}.
 \begin{figure}[!htbp]
  \centering
\includegraphics[width=\textwidth]{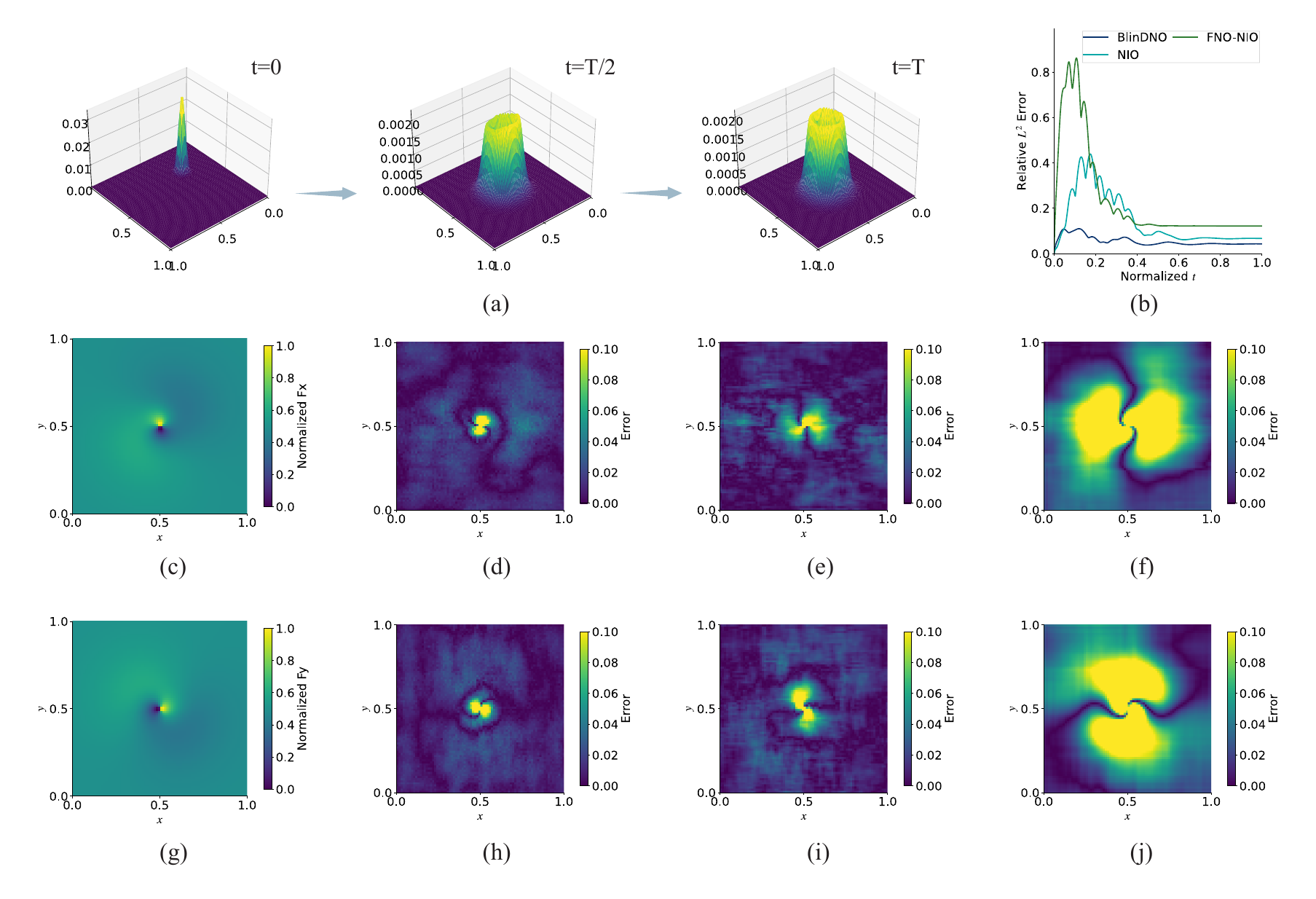}
  \caption{Inversion results for the 2D Fokker–Planck equation with a non-conservative force field. 
(a) Temporal evolution of the ground-truth density, $\rho_{\mathrm{GT}}(x,t)$. 
(b) Relative $L_{2}$ error at each time step between the simulated density and the ground truth. 
(c) $x$-component of the ground-truth force field, $F_x(x)$. 
(d)–(f) Error of the reconstructed $x$-component of the force field, $F_x(x)$, obtained using \textsc{BlinDNO}, NIO, and FNO–NIO. 
(g) $y$-component of the ground-truth force field, $F_y(x)$. 
(h)–(j) Error of the reconstructed $y$-component of the force field, $F_y(x)$, obtained using \textsc{BlinDNO}, NIO, and FNO–NIO.
}
  \label{fig:5}
\end{figure}

\begin{table}[t]
\centering
\label{tab:1}
\begin{tabular}{lcccccc}
\toprule
\multirow{2}{*}{Model} & \multicolumn{2}{c}{1D FPE} & \multicolumn{2}{c}{1D Schrödinger} & \multicolumn{2}{c}{1D GPE} \\
 & $\boldsymbol{E}_{\theta}\downarrow$ & $\boldsymbol{E}_{\boldsymbol{\rho}}\downarrow$ 
 & $\boldsymbol{E}_{\theta}\downarrow$ & $\boldsymbol{E}_{\boldsymbol{\rho}}\downarrow$
 & $\boldsymbol{E}_{\theta}\downarrow$ & $\boldsymbol{E}_{\boldsymbol{\rho}}\downarrow$ \\
\midrule
NIO          & 14.1 & 25.9 & 4.2 & 26.1 & 4.1 & 15.0 \\
FNO-NIO      & 17.1 & 10.9 & 4.8 & 20.5 & 2.8 & 7.3 \\
\textsc{BlinDNO}      & \textbf{12.1} & \textbf{10.1} & \textbf{3.8} & \textbf{18.2} & \textbf{2.7} & \textbf{6.4} \\
\midrule
\multirow{2}{*}{} & \multicolumn{2}{c}{2D Diffusion} & \multicolumn{2}{c}{2D Force} & \multicolumn{2}{c}{} \\
 & $\boldsymbol{E}_{\theta}\downarrow$ & $\boldsymbol{E}_{\boldsymbol{\rho}}\downarrow$
 & $\boldsymbol{E}_{\theta}\downarrow$ & $\boldsymbol{E}_{\boldsymbol{\rho}}\downarrow$ & & \\
\midrule
NIO          & 5.8 & 41.9 & 4.9 & 12.7 \\
FNO-NIO      & 2.8 & 13.6 & 6.7 & 10.8 \\
\textsc{BlinDNO}      & \textbf{1.9} & \textbf{10.3} & \textbf{4.4} & \textbf{8.7} \\
\bottomrule
\end{tabular}
\caption{Comparison of  $E_{\textbf{param}}$ and $E_{\boldsymbol{\rho}}$ of NIO, FNO-NIO and \textsc{BlinDNO} on different tasks. 
Lower is better $\downarrow$. Best results are in \textbf{bold}.}
\end{table}
\subsection{Example 5: 3D Dynamics of Protein Folding}
\label{subsec:3d_protein_folding}
We conclude with a high-dimensional inverse problem arising from a realistic 3D cryo-EM setting. Specifically, we study the folding dynamics of the heterodimeric ABC exporter TmrAB, initially captured in an inward-facing conformation under turnover conditions, as reported in the EMDB entry EMD--4774\cite{hofmann2019conformation}. In our model, the left polypeptide chain remains fixed, whereas the right chain undergoes rigid-body rotation about a prescribed axis passing through a point above the protein complex.

The probability density $p(\mathbf{x},t)$ evolves according to the 3D FPE where the drift field is induced by the rotational velocity
\[
\boldsymbol{\mu}(\mathbf{x})
= \omega\,(\mathbf{u}\times\mathbf{r})\,\mathbb{I}_{\Omega_{\mathrm{right}}},
\]
with $\mathbf{u}$ denoting the unit rotation axis, $\omega$ the angular velocity, and $\mathbf{r}$ the position vector. For each simulation, the rotation axis $\mathbf{u}$ is drawn uniformly from the unit sphere, and the angular velocity $\omega$ is sampled from $\mathrm{Unif}[0.5,2.0]$. The velocity is set to zero inside $\Omega_{\mathrm{left}}$, ensuring that the left chain remains stationary. The computational domain is $\Omega=[0,10]^3$, and the diffusion coefficient is fixed at $D = 10^{-4}$. The initial density $\rho(\mathbf{x},0)$ is obtained from the cryo-EM map of EMD--4774 and downsampled by a factor of eight, yielding a grid with spatial resolution $\Delta x = 10/64$.

Spatial derivatives for advection and diffusion are approximated using fourth-order central finite differences, and time integration is performed using a third-order SSP–RK scheme. Reflective boundary conditions are imposed to ensure mass conservation, and negative values introduced by numerical errors are truncated and renormalized. The final time is $T_{\mathrm{final}} = 2.0$ with a maximum time step $\Delta t_{\max} = 0.002$. This procedure produces a dataset of $1600$ training samples and $400$ testing samples. This experiment constitutes a demanding benchmark due to its high spatial dimensionality, the nontrivial rotational dynamics, and the complex initial density derived directly from experimental cryo-EM measurements.

After training, \textsc{BlinDNO} accurately reconstructs the high-dimensional velocity field, successfully identifying both the direction of the rotation axis and the corresponding angular velocity (see Figs.~\ref{fig:6}(g)–(l)). Using the reconstructed velocity field, we simulate the resulting protein folding dynamics. The predicted trajectories (Figs.~\ref{fig:6}(d)–(f)) closely track the ground-truth evolution observed in Figs.~\ref{fig:6}(a)–(c), demonstrating that \textsc{BlinDNO} is capable of capturing the high-dimensional dynamics governing by complex 3D conformational transition.

 \begin{figure}[!htbp]
  \centering
\includegraphics[width=\textwidth]{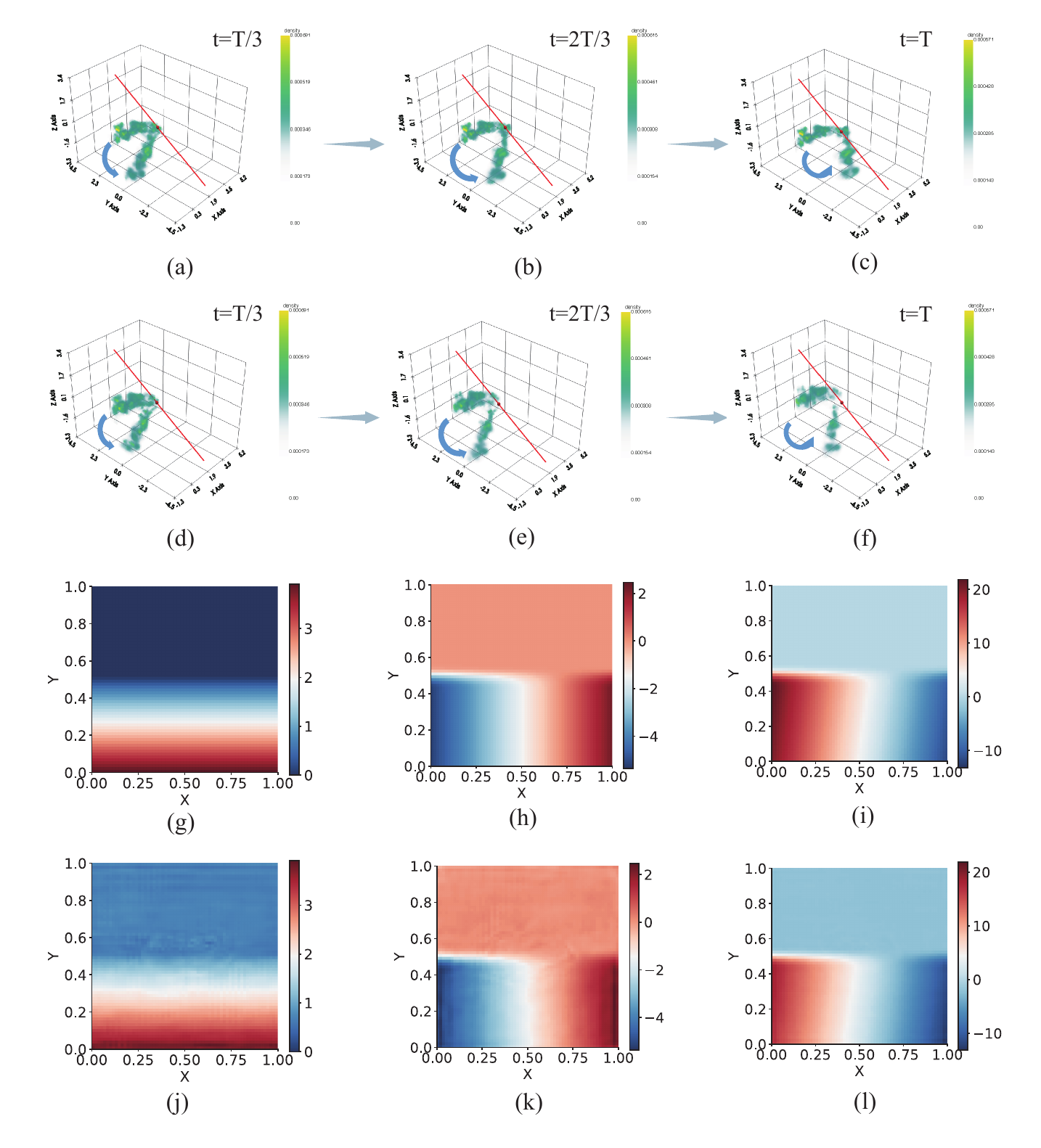}
  \caption{Inversion results for the 3D protein folding dynamics. 
(a)-(c) Temporal evolution of the ground-truth density, $\rho_{\mathrm{GT}}(x,t)$ at $t = \frac{T}{3},\frac{2T}{3},T$. 
(d)-(f) Temporal evolution of the reconstructed density, $\rho_{\mathrm{GT}}(x,t)$ at the same time points. 
(g)–(i) Ground-truth velocity components $\boldsymbol{\mu}_x(x,y,0)$,$\boldsymbol{\mu}_y(x,y,0)$,$\boldsymbol{\mu}_z(x,y,0)$. 
(j)–(l) Reconstructed velocity components $\boldsymbol{\mu}_x(x,y,0)$,$\boldsymbol{\mu}_y(x,y,0)$,$\boldsymbol{\mu}_z(x,y,0)$. 
}
  \label{fig:6}
\end{figure}

\section{Conclusions and discussions}\label{sec:conclusion}
This work extends the classical inverse problem of parameter reconstruction in probability evolution dynamics to a time-label-free setting. We presented a unified mathematical formulation for recovering unknown drift and diffusion terms in SDE-driven dynamics, as well as external potentials in quantum mechanical systems, using only unordered observations of evolving probability densities. To address the intrinsic challenges posed by the loss of temporal information, we introduced \textsc{BlinDNO}, an efficient distribution-to-functions neural operator designed to directly approximate the inverse operator associated with these problems.

The proposed framework leverages a permutation-invariant architecture tailored for unordered density snapshots. Central to this design is an attention-based UNet mixer coupled with a high-order Janossy pooling mechanism, enabling the extraction of compact yet expressive distributional features. These features are subsequently mapped to functional outputs via a lightweight Fourier neural operator, resulting in a scalable and robust distribution-to-functions pipeline. Through a sequence of experiments---ranging from one-dimensional SDE and quantum dynamics to two-dimensional nonequilibrium systems and a three-dimensional synthetic cryo-EM protein folding model---we demonstrated that \textsc{BlinDNO} consistently outperforms existing baselines and achieves accurate reconstructions across diverse dynamical regimes and spatial dimensionalities, despite operating in a time-label-free setting.

There are several important directions for future work. First, it is of theoretical interest to further analyze the properties of the inverse operators considered here and to establish approximation guarantees for the proposed distribution-to-functions neural operator. Such results may provide additional insight into the expressiveness and stability of \textsc{BlinDNO}. Second, while data-driven operator learning methods often exhibit limited generalization to out-of-distribution regimes, one promising avenue is to integrate \textsc{BlinDNO} into PDE-constrained optimization frameworks. In this setting, \textsc{BlinDNO} could serve as a fast initializer, significantly accelerating the solution of large-scale inverse problems and enabling applications to real cryo-EM experimental datasets. These extensions will be explored in future work.


\bibliographystyle{siamplain}
\bibliography{references}
\end{document}